\documentclass[letterpaper,10pt,conference]{ieeeconf}

\IEEEoverridecommandlockouts	%
\makeatletter

\let\proof\@undefined
\let\endproof\@undefined
\makeatother

\usepackage[abs]{overpic}
\usepackage{amsmath,amssymb,amsthm}
\usepackage{graphicx,url}
\usepackage{epsfig,psfrag,color}
\usepackage{xspace,subfig}
\usepackage{caption}

\usepackage{caption}

\usepackage[linesnumbered,vlined,ruled]{algorithm2e}
\usepackage{multirow}
\usepackage[normalem]{ulem} %
\usepackage{cancel}

\newtheorem{theorem}{Theorem}[section]
\newtheorem{proposition}[theorem]{Proposition}

\newtheorem{definition}[theorem]{Definition}

\newtheorem{remark}[theorem]{Remark}

\newtheorem{example}[theorem]{Example}

\newtheorem{problem}[theorem]{Problem}
\newcommand{\exampler}[2]{\hskip -\marginparsep {\bf Example #1 Revisited.} {\it #2}}

\newcommand{\RMB}{\mathrm{B}}
\newcommand{\RMC}{\mathrm{C}}
\newcommand{\RMR}{\mathrm{R}}
\newcommand{\RMT}{\mathrm{T}}
\newcommand{\RMW}{\mathrm{W}}

\newcommand{\CF}{\mathcal{F}}

\newcommand{\CL}{\mathcal{L}}

\newcommand{\CQ}{\mathcal{Q}}

\newcommand{\CS}{\mathcal{S}}

\newcommand{\BFB}{\mathbf{B}}

\newcommand{\BFR}{\mathbf{R}}

\newcommand{\BFT}{\mathbf{T}}

\graphicspath{{fig/}}

\newcommand{\BBR}{\mathbb{R}}
\newcommand{\BBT}{\mathbb{T}}

\newcommand{\project}[1]{\! \upharpoonright_{#1}}

\newcommand{\andltl}{\wedge}

\newcommand{\Next}{\mathbf{X}}
\newcommand{\Always}{\mathbf{G}}
\newcommand{\Event}{\mathbf{F}}
\newcommand{\Until}{\mathcal{U}}
\newcommand{\Implies}{\Rightarrow}
\newcommand{\Not}{\lnot}

\newcommand{\prop}{\alpha}
\newcommand{\opt}{\pi}
\newcommand{\optrun}{\textsc{Optimal-Run}\ }

\newcommand{\constR}{\textsc{Obtain-Region-Automaton}\ }
\newcommand{\serializeR}{\textsc{Serialize-Region-Automaton}\ }

\newcommand{\ie}{{\it i.e.},\;}

\newcommand\oprocendsymbol{\hbox{$\square$}}
\newcommand\oprocend{\relax\ifmmode\else\unskip\hfill\fi\oprocendsymbol}

\urldef{\ulusoy}\url{alphan@bu.edu}
\urldef{\ding}\url{dingx@utrc.utc.com}
\urldef{\smith}\url{stephen.smith@uwaterloo.ca}
\urldef{\belta}\url{cbelta@bu.edu}
\title{\LARGE \bf Robust Multi-Robot Optimal Path Planning with\\Temporal Logic Constraints}
\author{Alphan Ulusoy$^\star$%
 		\quad Stephen L. Smith$^\dagger$%
 		\quad Xu Chu Ding$^\ddagger$%
 		\quad Calin Belta$^\star$%
 		\thanks{This work was supported in part by ONR-MURI N00014-09-1051, ARO W911NF-09-1-0088, AFOSR YIP FA9550-09-1-020, NSF CNS-0834260, NSF CNS-1035588 and NSERC.}%
 		\thanks{$^\star$ Hybrid and Networked Systems Laboratory, Boston University, Boston, MA 02215 (\ulusoy, \belta)}%
 		\thanks{$^\dagger$ Dept. of Electrical and Computer Engineering, University of Waterloo, Waterloo ON, N2L 3G1 Canada (\smith)}
		\thanks{$^\ddagger$ Embedded Systems and Networks, United Technologies Research Center, East Hartford, CT 06108 (\ding)}}%

\begin{document}

\maketitle
\thispagestyle{empty}
\pagestyle{empty}

\begin{abstract}
In this paper we present a method for automatically planning robust optimal paths for a group of robots that satisfy a common high level mission specification. 
Each robot's motion in the environment is modeled as a weighted transition system, and the mission is given as a Linear Temporal Logic (LTL) formula over a set of propositions satisfied by the regions of the environment.
In addition, an optimizing proposition must repeatedly be satisfied.
The goal is to minimize the maximum time between satisfying instances of the optimizing proposition while ensuring that the LTL formula is satisfied even with uncertainty in the robots' traveling times.
We characterize a class of LTL formulas that are robust to robot timing errors, for which we generate optimal paths if no timing errors are present, and we present bounds on the deviation from the optimal values in the presence of errors.
We implement and experimentally evaluate our method considering a persistent monitoring task in a road network environment.
\end{abstract}

\section{Introduction\label{sec:introduction}} 
The classic motion planning problem considers missions where a robot must reach a goal state from an initial state while avoiding obstacles.
Temporal logics, on the other hand, provide a powerful high-level language for specifying complex missions for groups of robots~\cite{Loizou04, HKG-GEF-GJP:09, TW-UT-RMM:10, AB-LEK-MYV:10, LB-OS-JC-KG-SL:11}.
Their power lies in the wealth of tools from model checking~\cite{VW86, Holzmann97}, which can be leveraged to generate robot paths satisfying desired mission specifications.
Alternatively, if the mission cannot be satisfied, the tools can be used produce a certificate, or counter-example, which proves that the mission is not possible.
However, in robotics the goal is typically to plan paths that not only complete a desired mission, but which do so in an optimal manner.
In our earlier work~\cite{SLS-JT-CB-DR:10b-IJRR} we considered Linear Temporal Logic (LTL) specifications, and a particular form of cost function, and provided a method for computing optimal robot paths for a single robot.
We then extended this approach to multi-robot problems by utilizing timed automata~\cite{iros2011}.

The main difficulty in moving from a single robot to multiple robots is in synchronizing the motion of the robots, or in allowing the robots to move asynchronously.
In~\cite{MK-CB:08}, the authors propose a method for decentralized motion of multiple robots by restricting the robots to take transitions (\ie travel along edges in the graph) synchronously.
Once every robot has completed a transition, the robots can synchronously make the next transition.
While such an approach is effective for satisfying the LTL formula, it does not lend itself to optimizing the robot motion, since robots must spend extra time for synchronization.
In~\cite{iros2011} we approached this problem by describing the motion of the group of robots in the environment as a timed automaton.
This description allowed us to represent the relative position between robots. Such information is necessary for optimizing the robot motion.
After providing a bisimulation~\cite{Milner89} of the infinite-dimensional timed automaton to a finite dimensional transition system we were able to apply our results from~\cite{SLS-JT-CB-DR:10b-IJRR} to compute an optimal run.

However, enabling the asynchronous motion of robots introduces issues in the robustness, and thus implementability of the multi-robot paths.
Timed-automata rely heavily on the assumption that the clocks (or for robots, the speeds), are known exactly.
If the clocks drift by even an infinitesimally small amount, then the reachability analysis developed for timed-automata is no longer correct~\cite{AP:99,AP:00}.
The intuition behind this is that if the robot speeds are not exactly equal to those used for planning, then two robots can complete tasks in a different order than was specified in the plan.
This switch in the order of events may result in the violation of the global mission specification.

In this paper, we address this issue by characterizing a class of LTL formulas that are robust to such timing errors.
For simplicity of presentation, we assume that each robot moves among the vertices of an environment modeled as a graph.
However, by using feedback controllers for facet reachability in polytopes~\cite{HS04} the method developed in this paper can be extended to robots with continuous dynamics traversing an environment with polytopic partitions.
The characterization relies on the concept of trace-closedness of languages, which was first applied in multi-robot planning in~\cite{YC-XCD-CB:11}.
For these languages, we can guarantee that any deviation from the planned order of events due to uncertainties in the speeds of robots will not result in the violation of the global specification.

The contribution of this paper is to present a method for generating paths for a group of robots satisfying general LTL formulas, which are robust to uncertainties in the speeds of robots, and which perform within a known bound of the optimal value.
We focus on minimizing a cost function that captures the maximum time between satisfying instances of an \emph{optimizing proposition}.
The cost is motivated by problems in persistent monitoring and in pickup and delivery problems.
Our solution relies on using the concept of trace-closedness to characterize the class of LTL formulas for which a robust solution exists.
For formulas in this class, we utilize a similar method as in~\cite{iros2011} to generate robot plans.
We then propose periodic synchronization of the robots to optimize the cost function in the presence of timing errors.
We provide results from an implementation on a robotic test-bed, which shows the utility of the approach in practice.

The organization of the paper is as follows. In Section~\ref{sec:preliminaries}, we give some preliminaries in formal methods and trace-closed languages.
In Section~\ref{sec:problem}, we formally state the motion planning problem for a team of robots, and we present our solution in Section~\ref{sec:solution}.
In Section~\ref{sec:experiments}, we present a hardware implementation for a team of robots performing persistent data gathering missions in a road network environment.
Finally, in Section~\ref{sec:conclusions}, we conclude with final remarks.

\section{Preliminaries \label{sec:preliminaries}}

For a set $\Sigma$, we use $|\Sigma|$, $2^\Sigma$, $\Sigma^*$, and $\Sigma^\omega$ to denote its cardinality, power set, set of finite words, and set of infinite words, respectively.
Moreover, we define $\Sigma^\infty = \Sigma^* \cup \Sigma^\omega$ and denote the empty string by $\emptyset$.

\begin{definition}[\bf Transition System]
\label{def:TS}
A (weighted) transition system (TS) is a tuple $\BFT := (\CQ_\RMT, q_\RMT^0, \delta_\RMT, \Pi_\RMT, \CL_\RMT, w_\RMT)$, where
\begin{enumerate}
\item $\CQ_\RMT$ is a finite set of states; %
\item $q_\RMT^0 \in \CQ_\RMT$ is the initial state; %
\item $\delta_\RMT \subseteq \CQ_\RMT \times \CQ_\RMT$ is the transition relation; %
\item $\Pi_\RMT$ is a finite set of atomic propositions (observations); %
\item $\CL_\RMT:\CQ_\RMT\to 2^{\Pi_\RMT}$ is a map giving the set of atomic propositions satisfied in a state; %
\item $w_\RMT: \delta_{T}\to \BBR_{>0}$ is a map that assigns a positive weight to each transition.
\end{enumerate}
\end{definition}

We define a run of $\BFT$ as an infinite sequence of states $r_\RMT = q^0q^1\ldots$ such that $q^0 = q^{0}_{T}$, $q^k \in \CQ_\RMT$ and $(q^k,q^{k+1}) \in \delta_\RMT$ for all $k \geq 0$.
A run generates an infinite word $\omega_\RMT = \CL(q^0)\CL(q^1)\ldots$ where $\CL(q^k)$ is the set of atomic propositions satisfied at state $q^k$.  

\begin{definition}[\bf LTL Formula]
An LTL formula $\phi$ over the atomic propositions $\Pi$ is defined inductively as follows:
\begin{equation*}
\phi ::= \top \mid \prop \mid \phi \lor \phi \mid \phi \andltl \phi \mid \lnot\,\phi \mid \Next\,\phi \mid \phi\,\Until\,\phi
\end{equation*}
where $\top$ is a predicate true in each state of a system, $\prop \in \Pi$ is an atomic proposition, $\neg$ (negation), $\vee$ (disjunction) and $\wedge$ (conjunction) are standard Boolean connectives, and $\Next$ and $\Until$ are temporal operators.
\end{definition}
LTL formulas are interpreted over infinite words (generated by the transition system $\BFT$ from Def.~\ref{def:TS}).
Informally, $\Next\,\prop$ states that at the next position of a word, proposition $\prop$ is true.
The formula $\prop_1\,\Until\,\prop_2$ states that there is a future position of the word when proposition $\prop_2$ is true, and proposition $\prop_1$ is true at least until $\prop_2$ is true.
From these temporal operators we can construct two other temporal operators: Eventually (i.e., future), $\Event$ defined as $\Event\,\phi := \top\,\Until\, \phi$, and Always (i.e., globally), $\Always$, defined as $\Always\,\phi := \lnot\,\Event\,\lnot\,\phi$.
The formula $\Always\,\phi$ states that $\phi$ is true at all positions of the word; the formula $\Event\,\phi$ states that $\phi$ eventually becomes true in the word.
More expressivity can be achieved by combining the temporal and Boolean operators.
We say a run $r_\RMT$ satisfies $\phi$ if and only if the word generated by $r_\RMT$ satisfies $\phi$.

\begin{definition}[\bf B\"uchi Automaton]
\label{def:buchi}
  A B\"uchi automaton is a tuple $\BFB := (\CS_\RMB,\CS_\RMB^0,\Sigma_\RMB,\delta_\RMB,\CF_\RMB)$, consisting of %
\begin{enumerate}
\item a finite set of states $\CS_\RMB$; %
\item a set of initial states $\CS_\RMB^0\subseteq \CS_\RMB$; %
\item an input alphabet $\Sigma_\RMB$; %
\item a non-deterministic transition relation $\delta_\RMB \subseteq \CS_\RMB\times \Sigma_\RMB \times \CS_\RMB$; %
\item a set of accepting (final) states $\CF_\RMB\subseteq \CS_\RMB$.
\end{enumerate}
\end{definition}

A \emph{run} of $\BFB$ over an input word $\omega=\omega^0\omega^1\ldots$ is a sequence $r_\RMB=s^0s^1\ldots$, such that $s^0 \in \CS_\RMB^0$, and $(s^k,\omega^k,s^{k+1}) \in \delta_\RMB$, for all $k\geq 0$.
A B\"{u}chi automaton $\BFB$ accepts a word over $\Sigma_\RMB$ if and only if at least one of the corresponding runs intersects with $\CF_\RMB$ infinitely many times.
For any LTL formula $\phi$ over a set $\Pi$, one can construct a B\"{u}chi automaton with input alphabet $\Sigma_\RMB = 2^{\Pi}$ accepting all and only words over $2^\Pi$ that satisfy $\phi$.

\begin{definition}[\bf Prefix-Suffix Structure]
\label{def:prefix_suffix}
A prefix of a run is a finite path from an initial state to a state $q$.
A periodic suffix is an infinite run originating at the state $q$ reached by the prefix, and periodically repeating a finite path, which we call the suffix cycle, originating and ending at $q$, and containing no other occurrence of $q$.
A run is in prefix-suffix form if it consists of a prefix followed by a periodic suffix.
\end{definition}

\begin{definition}[\bf Language]
The set of all the words accepted by an automaton $\BFB$ is called the language recognized by the automaton and is denoted by $L_\RMB$.
\end{definition}

\begin{definition}[\bf Distribution]
\label{def:dist}
Given a set $\Sigma$, the collection of subsets $\Sigma_i \subseteq \Sigma$, $\forall \; i=1,\ldots,m$ is called a distribution of $\Sigma$ if $\cup_{i=1}^m\Sigma_i = \Sigma$.
\end{definition}

\begin{definition}[\bf Projection]
For a word $\omega \in \Sigma^\infty$ and a subset $\Sigma_i \subseteq \Sigma$, $\omega \project{\Sigma_i}$ denotes the projection of $\omega$ onto $\Sigma_i$, which is obtained by removing all the symbols in $\omega$ that are not in $\Sigma_i$.
For a language $L\subseteq\Sigma^\infty$ and a subset $\Sigma_i \subseteq \Sigma$, $L \project{\Sigma_i}$ denotes the projection of $L$ onto $\Sigma_i$, which is the set of projections of all words in $L$ onto $\Sigma_i$, \ie $\{\omega \project{\Sigma_i} |\ \omega \in L\}$.
\end{definition}

\begin{definition}[\bf Trace-Closed Language]
\label{def:trace_closed_lang}
Given the distribution $\{\Sigma_1,\ldots,\Sigma_m\}$ of $\Sigma$ and the words $\omega, \omega' \in \Sigma^\infty$, $\omega'$ is trace-equivalent to $\omega$, denoted $\omega' \sim \omega$, iff their projections onto each one of the subsets in the given distribution are equal, \ie $\omega \project{\Sigma_i} = \omega' \project{\Sigma_i}$ for each $i=1,\ldots,m$.
For $\{\Sigma_1,\ldots,\Sigma_m\}$, the trace-equivalence class of $\omega$ is given by $[\omega] =\{\omega' \in \Sigma^\infty\;|\;\omega'\project{\Sigma_i} = \omega\project{\Sigma_i} \forall i=1,\ldots,m\}$.
Finally, a trace-closed language over $\{\Sigma_1,\ldots,\Sigma_m\}$ is a language $L$ such that $[\omega] \subseteq L,\,\forall\ \omega \in L$.
\end{definition}

\section{Problem Formulation and Approach \label{sec:problem}}

In this section we introduce the multi-robot path planning problem with temporal constraints, and we motivate the need for solutions that are robust to uncertain robot speeds.  

\subsection{Environment Model and Initial Formulation} 

Let
\begin{equation}\label{eqn:graph}
\mathcal{E}=(V,\rightarrow_\mathcal{E})
\end{equation}
be a graph, where $V$ is the set of vertices and $\rightarrow_{\mathcal{E}}\subseteq V\times V$ is the set of edges.
In this paper, $\mathcal{E}$ is the quotient graph of a partitioned environment, where $V$ is a set of labels for the regions in the partition and $\rightarrow_{\mathcal{E}}$ is the corresponding adjacency relation.
For example, $V$ can be a set of labels for the roads, intersections, and buildings in an urban-like environment and $\rightarrow_\mathcal{E}$ gives their connections (see Fig.~\ref{fig:road_network}). 

Consider a team of $m$ robots moving in an environment modeled by $\mathcal E$.
The motion capabilities of robot $i \in \{1,\ldots,m\}$ are represented by a TS $\BFT_i = (\CQ_i,q_i^0,\delta_i,\Pi_i,\CL_i,w_i)$, where $\CQ_{i}\subseteq V$; $q_i^0$ is the initial vertex of robot $i$; $\delta_i\subseteq \rightarrow_{\mathcal E}$ is a relation modeling the capability of robot $i$ to move among the vertices; $\Pi_i$ is the subset of propositions $\Pi$ assigned to the environment that can be satisfied by robot $i$ such that $\{\Pi_1,\ldots,\Pi_m\}$ is a distribution of $\Pi$; $\CL_i$ is a mapping from $\CQ_i$ to $2^{\Pi_i}$ showing how the propositions are satisfied at vertices; $w_i(q,q')$ captures the time for robot $i$ to go from vertex $q$ to $q'$, which we assume to be an integer.
In this robotic model, robot $i$ travels along the edges of $\BFT_{i}$, and spends zero time on the vertices.
We assume that the robots are equipped with motion primitives which allow them to move from $q$ to $q'$ for each $(q,q')\in \delta_i$.

In our previous work~\cite{iros2011} we considered the case where there is an atomic proposition $\opt \in \Pi$, called the \emph{optimizing proposition}, and a multi-robot task specified by an LTL formula of the form
\begin{equation}
\label{eqn:general_formula_old}
\phi:=\varphi \land \Always\Event\opt,
\end{equation}
where $\varphi$ can be any LTL formula over $\Pi$, and $\Always\Event\opt$ specifies that proposition $\opt$ must be satisfied infinitely often.
As an example, in a persistent data gathering task, $\opt$ may be assigned to regions where data is uploaded, \ie $\pi=\mathtt{Upload}$, while $\varphi$ can be used to specify rules (such as traffic rules) that must be obeyed at all times during the task~\cite{SLS-JT-CB-DR:10b-IJRR}.  

Our goal in~\cite{iros2011} was to plan multi-robot paths that satisfy $\phi$ and minimize the maximum time between satisfying instances of $\pi$.  In data gathering, this corresponds to minimizing the maximum time between data uploads.
To state this problem formally, we assume that each run $r_i = q_i^0q_i^1\ldots$ of $\BFT_i$ (robot $i$) starts at $t=0$ and generates a word $\omega_i = \omega_i^0\omega_i^1\ldots$ and a corresponding sequence of time instances $\BBT_i := t_i^0t_i^1\ldots$ such that the $k^{th}$ symbol $\omega_i^k = \CL_i(q_i^k)$ is satisfied at time $t_i^k$.
Note that, as robots spend zero time on the vertices, each $\omega_i^k$ has a unique $t_i^k$ which is the instant when robot $i$ visits the corresponding vertex.
To define the behavior of the team as a whole, we consider the sequences $\BBT_{i}$ as sets and take the union $\bigcup_{i=1}^{m} \BBT_{i}$ and order this set in ascending order to obtain $\BBT := t^0t^1,\ldots$.
Then, we define $\omega_{team} = \omega_{team}^0\omega_{team}^1\ldots$ to be the word generated by the team of robots where the $k^{th}$ symbol $\omega_{team}^k$ is the union of all propositions satisfied at time $t^k$.
Finally, we define the infinite sequence $\BBT^\opt = \BBT^\opt(1),\BBT^\opt(2),\ldots$ where $\BBT^\opt(k)$ stands for the time instance when the optimizing proposition $\opt$ is satisfied for the $k^{th}$ time by the team.
Thus, the problem is that of synthesizing individual optimal runs for a team of robots so that $\omega_{team}$ satisfies $\phi$ and $\BBT^\pi$ minimizes
\begin{equation}
\label{eqn:cost_function}
J(\BBT^{\opt})=\limsup_{k\to+\infty}\left(\mathbb{T}^{\opt}(k+1) - \mathbb{T}^{\opt}(k)\right).
\end{equation}

Since we consider LTL formulas containing $\Always\Event \pi$, this optimization problem is always well-posed.

\subsection{Robustness and Optimality in the Field}
In this paper, we are interested in the implementability of our previous approach in the case where our model is not exact in the weights of transitions.
Particularly, we consider the case where the actual value of $w_i(q,q')$ that is observed during deployment, denoted by $\tilde w_i(q,q')$, is a non-deterministic quantity that lies in the interval $[(1-\rho_i)w_i(q,q'), (1+\rho_i)w_i(q,q')]$ where $\rho_i$ is the \emph{deviation value} of robot $i$ which is assumed to be known a priori.
In the following, we use the expression ``\emph{in the field}'' to refer to the model with uncertain traveling times, and use $x$ and $\tilde x$ to denote the planned and actual values of some variable $x$.

The question becomes, if we use the runs generated from our previous approach in the field, will the formula $\phi$ still be satisfied?  
Given the word $\omega_{team}$ that characterizes the planned run of the robotic team and the distribution $\{\Pi_1,\ldots,\Pi_m\}$, the \emph{actual} word $\tilde\omega_{team}$ generated by the robotic team during its infinite asynchronous run in the field will be one of the trace equivalents of $\omega_{team}$, \ie  $\tilde\omega_{team} \in [\omega_{team}]$ due to the uncertainties in the traveling times of the robots.
This leads to the definition of critical words.

\begin{definition}[\bf Critical Words] 
\label{def:critical_words} 
Given the language $L_\RMB$ of the B\"uchi automaton that corresponds to the LTL formula $\phi$ over $\Pi$, and given a distribution of $\Pi$, we define the word $\omega$ over $\Pi$ to be a \emph{critical word} if $\exists\;\tilde\omega \in [\omega]$ such that $\tilde\omega \not\in L_\RMB$.  
\end{definition}

Thus, we see that if the planned word is critical, then we may not satisfy the specification in the field.
This can be formalized by noting that the optimal runs that satisfy \eqref{eqn:general_formula_old} are always in a prefix-suffix form \cite{SLS-JT-CB-DR:10b}, where the suffix cycle is repeated infinitely often.
Using this observation and Def.~\ref{def:critical_words} we can formally define the words that can violate the LTL formula during the deployment of the robotic team.

\begin{proposition}
\label{prp:critical_suffix}
If the suffix cycle of the word $\omega_{team}$ is a critical word, then the correctness of the motion of the robotic team during its deployment cannot be guaranteed.
\end{proposition}
\begin{proof}
We denote the actual word generated by the robotic team in the field by $\tilde\omega_{team}$ whereas $\omega_{team}$ stands for the planned word.
Suppose that for each robot $\rho_i=\epsilon$, and in the suffix cycle we have $\omega_{team}^k$ and $\omega_{team}^{(k+\tau)}$ generated by robots $i$ and $j$ at positions $k$ and $k+\tau$ that must not be swapped, because if they do $\tilde\omega_{team}$ violates $\phi$.
Note that we are guaranteed to find such symbols as we assume the suffix cycle to be a critical word.
In the worst-case, for the symbols to swap, we must have $(1 + \epsilon)t^k > (1-\epsilon)t^{k+\tau}$.
Solving for $\epsilon$, we get $\epsilon > (t^{k+\tau}-t^k)/(t^k+t^{k+\tau})$.
However, as the suffix is an infinite repetition of the suffix cycle, $\lim_{k\rightarrow\infty}(t^{k+\tau}-t^k)/(t^k+t^{k+\tau}) = 0$ and $\phi$ is violated for any $\epsilon>0$. 
\end{proof}

In addition, we can consider the performance of the team during deployment in terms of the value of the cost function \eqref{eqn:cost_function} observed in the field.
Using the same arguments presented in Prop.~\ref{prp:critical_suffix} it can be easily show that, the worst-case field value of \eqref{eqn:cost_function} will be the minimum of $(\tilde J_1, \ldots, \tilde J_m)$ where $\tilde J_i$ is the maximum duration between any two successive satisfactions of $\pi$ by robot i in the field.
This effectively means that there is no benefit in executing the task with multiple robots, as at some point in the future the overall performance of the team will be limited by that of a single member.

\subsection{Robust Problem Formulation}

To characterize the field performance of the robotic team and to limit the deviation from the optimal run during deployment, we propose to use a synchronization protocol where robots can synchronize with each other only when they are at the vertices of the environment.
We assume that there is an atomic proposition $\mathtt{Sync} \in \Pi$, called the \emph{synchronizing proposition}, and we consider multi-robot tasks specified using LTL formulas of the form
\begin{equation}
\label{eqn:general_formula}
\phi_{sync}:=\varphi\andltl\Always\Event\opt\andltl\Always\Event\mathtt{Sync},
\end{equation}
where $\varphi$ can be any LTL formula over $\Pi$, $\pi$ is the optimizing proposition and $\mathtt{Sync}$ is the special synchronizing proposition that is satisfied only when all members of the robotic team occupy vertices at the same time.
We can now formulate the problem.
\begin{problem}
\label{prb:problem}
Given a team of $m$ robots modeled as transition systems $\BFT_{i}$, $i\!=\!1,\ldots,m$, and an LTL formula $\phi_{sync}$ over $\Pi$ in the form \eqref{eqn:general_formula}, synthesize individual runs $r_i$ for each robot such that $\BBT^\pi$ minimizes the cost function \eqref{eqn:cost_function}, and $\tilde\omega_{team}$, \ie the word observed in the field, satisfies $\phi_{sync}$.
\end{problem}

Note that the runs produced by a solution to Prob. \ref{prb:problem} are guaranteed not to violate $\phi_{sync}$ even if there is a mismatch between the weights $w_i(q,q')$ used for the solution of the problem and the actual traveling times observed in the field.
Since $\tilde\omega_{team}$ observed in the field is likely to be sub-optimal, we will also seek to bound the deviation from optimality in the field.

\subsection{Solution Outline}
In \cite{iros2011}, we showed that the joint behavior of a robotic team can be captured by a region automaton.
A region automaton, as defined next, is a finite dimensional transition system that captures the relative positions of the members of the robotic team. 
This information is then used for computing optimal trajectories.
\begin{definition}[\bf Region Automaton]
\label{def:RA}
The region automaton $\BFR$ is a TS (Def.~\ref{def:TS}) $\BFR := (\CQ_\RMR,q_\RMR^0,\delta_\RMR, \Pi_\RMR, \CL_\RMR, w_\RMR)$, where
 	\begin{enumerate}
 	\item $\CQ_\RMR$ is the set of states of the form $(q,r)$ such that
	\begin{enumerate}
	\item $q$ is a tuple of state pairs $(q_1q_1',\ldots,q_mq_m')$ where the $i^{th}$ element $q_iq_i'$ is a source-target state pair from $\CQ_i$ of $\BFT_i$ meaning robot $i$ is currently on its way from $q_i$ to $q_i'$, and
	\item $r$ is a tuple of clock values $(x_1,\ldots,x_m)$ where the $i^{th}$ element denotes the time elapsed since robot $i$ left state $q_i$.
	\end{enumerate}
 	\item $q_\RMR^0$ is the initial state that has zero-weight transitions to all those states in $\CQ_\RMR$ with $r=(0,\ldots,0)$ and $q=(q_1^0q_1',\ldots,q_m^0q_m')$ such that $q_i^0$ is the initial state of $\BFT_i$ and $(q_i^0,q_i') \in \delta_i$.
 	\item $\delta_\RMR$ is the transition relation such that a transition from $(q,r)$ to $(q',r')$ exists if and only if
	\begin{enumerate}
	\item $(q_i,q_i'),(q_i',q_i'') \in \delta_i$ for all changed state pairs where the $i^{th}$ element $q_iq_i'$ in $q$ changes to $q_i'q_i''$ in $q'$,
	\item $w_i(q_i,q_i') - x_i$ of all changed state pairs are equal to each other and are strictly smaller than those of unchanged state pairs, and
	\item for all changed state pairs corresponding $x_i'$ in $r'$ becomes $x_i'=0$ and all other clock values in $r$ are incremented by $w_i(q_i,q_i')-x_i$ in $r'$.
	\end{enumerate}
	\item $\Pi_\RMR=\cup_{i=1}^m\Pi_i$ is the set of propositions;
	\item $\CL_\RMR:\CQ_\RMR\to 2^{\Pi_\RMR}$ is a map giving the set of atomic propositions satisfied in a state. For a state with $q=(q_1q_1',\ldots,q_mq_m')$, $\CL_\RMR((q,r)) = \cup_{i=1}^m\CL_i(q_i)$;
	\item $w_\RMR: \delta_{R}\to \BBR_{\geq0}$ is a map that assigns a non-negative weight to each transition such that $w_\RMR((q,r),(q',r'))=w_i(q_i,q_i')-x_i$ for each state pair that has changed from $q_iq_i'$ to $q_i'q_i''$ with a corresponding clock value of $x_i'=0$ in $r'$.
 	\end{enumerate}
\end{definition}

\begin{example}
Fig.~\ref{fig:ra} illustrates the region automaton $\BFR$ that corresponds to the robots modeled with $\BFT_1$ and $\BFT_2$ given in Fig.~\ref{fig:simple_ts}.
There is a transition from $((ba,bc),(0,0))$ to $((ba,cb),(1,0))$ with weight $1$ in $\BFR$ because $(b,c)\in\delta_2$, $w_2(b,c)=1$, and $w_1(b,a)\not=1$.
\end{example}

Our solution to Problem~\ref{prb:problem} can be outlined as follows:
\begin{enumerate}
\item We check if the LTL formula $\phi_{sync}$ is trace-closed guaranteeing that it will not be violated in the field (See Sec. \ref{sec:sub:dist}); 
\item We prepare the serialized region automaton of the robotic team with synchronization points by modifying the output of our earlier algorithm \constR \cite{iros2011} (See Sec. \ref{sec:sub:ra}); 
\item We find optimal runs on individual $\BFT_is$ using the \optrun algorithm we previously developed in \cite{SLS-JT-CB-DR:10b} and use a synchronization protocol to calculate an upper bound on the cost function \eqref{eqn:cost_function} for given deviation values to obtain the solution to Prob. \ref{prb:problem} (See Sec. \ref{sec:sub:soln}).
\end{enumerate}

\section{Problem Solution \label{sec:solution}}
In this section, we explain each step of the solution to Prob.~\ref{prb:problem} in detail.
In the following, we use a simple example to illustrate ideas as we develop the theory for the general case.
We present an experimental evaluation of our approach considering a more realistic scenario in Sec. \ref{sec:experiments}.

\subsection{Trace-Closedness of the Original Formula \label{sec:sub:dist}}

Prop.~\ref{prp:trace_closed_formula} shows how trace-closedness of $\phi_{sync}$ guarantees correctness in the field.
In the following, we say an LTL formula $\phi_{sync}$ is trace-closed if the language $L_\RMB$ of the corresponding B\"uchi automaton is trace-closed in the sense of Def.~\ref{def:trace_closed_lang}.

\begin{proposition}
\label{prp:trace_closed_formula}
If the general specification $\phi_{sync}$ is a trace-closed formula with respect to the distribution given by the robots' capabilities, then it will not be violated in the field due to uncertainties in the speeds of the robots.
\end{proposition}
\begin{proof}
From Defs.~\ref{def:trace_closed_lang} and \ref{def:critical_words}, we know that if we can find a run that satisfies a trace-closed LTL formula, then the word $\omega_{team}$ produced by the run will not be a critical word.
Since $\omega_{team}$ is not a critical word, $\not\exists$ $\tilde\omega_{team}\in[\omega_{team}]$ such that $\tilde\omega_{team}\not\in L_\RMB$.
Thus, regardless of the $\rho_i$ values of the robots, $\phi$ will not be violated in the field due to robot timing errors as any $\tilde\omega_{team}\in[\omega_{team}]$ will also be in $L_\RMB$.
\end{proof}

Thus, in order to guarantee correctness in the field, we first check that $\phi_{sync}$ is trace-closed using an algorithm adapted from~\cite{DP-TW-PW}.
However, as trace-closedness is not well-defined for words over $2^\Pi$, we construct a B\"uchi automaton whose language $L_\RMB$ is over the set $\Pi$.

\begin{example}
\label{exp:running_example}
Fig.~\ref{fig:simple_ts} illustrates the environment where two robots are expected to satisfy a task given by a formula in the form of \eqref{eqn:general_formula} where $\varphi = \Always \Event \mathtt{r1P} \andltl \Always \Event \mathtt{r2P}$, $\Pi_1=\{\mathtt{r1P},\,\pi,\mathtt{Sync}\}$, $ \Pi_2=\{\mathtt{r2P},\,\pi,\mathtt{Sync}\}$, and $\Pi = \{\mathtt{r1P},\,\mathtt{r2P},\,\pi,\,\mathtt{Sync}\}$.
\end{example}

\begin{figure}[h]
	\centering
	\subfloat[][]{\includegraphics[width=0.19\linewidth]{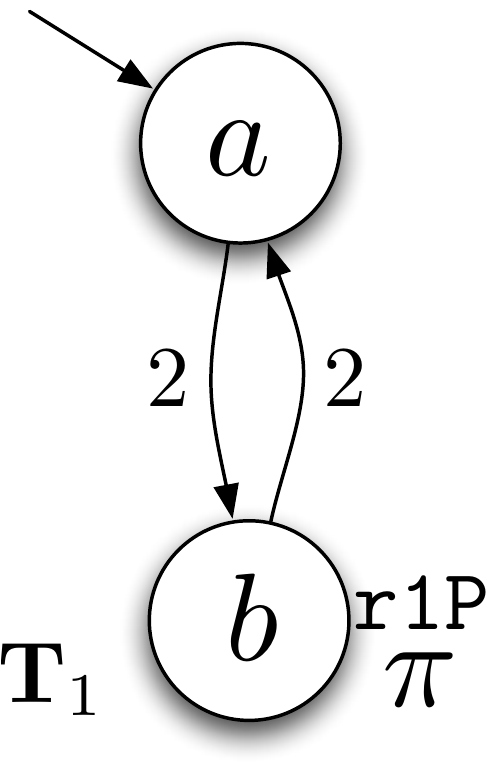}\label{fig:sub:ts_1}}\hspace{0.2in}
	\subfloat[][]{\includegraphics[width=0.3\linewidth]{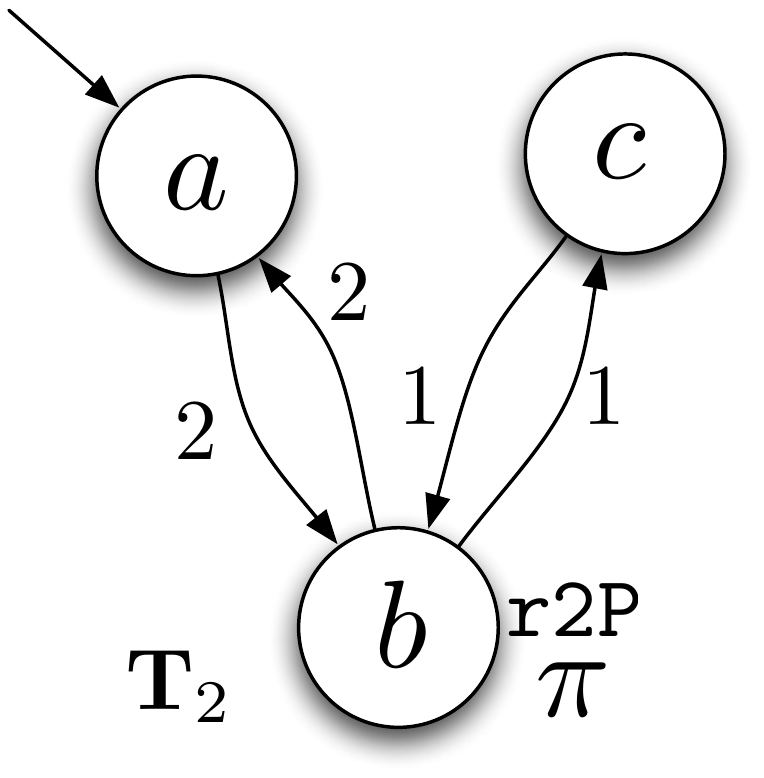}\label{fig:sub:ts_2}}
	\caption{TS's $\BFT_1$ and $\BFT_2$ of two robots in an environment with three vertices. The states correspond to vertices $\{a,b,c\}$, the edges represent the motion capabilities of each robot, and the weights represent the traveling times between any two vertices. The propositions $\mathtt{r1P, r2P}$ and $\pi$ are shown next to the vertices where they can be satisfied by the robots.}
	\label{fig:simple_ts}
\end{figure}

After checking that $\phi_{sync}$ is trace-closed, we proceed by obtaining the serialized region automaton with synchronization points where the $\mathtt{Sync}$ proposition is satisfied.

\subsection{Obtaining the Serialized Region Automaton with Synchronization Points \label{sec:sub:ra}}

\begin{figure}[h]
	\centering
	\includegraphics[width=0.95\linewidth]{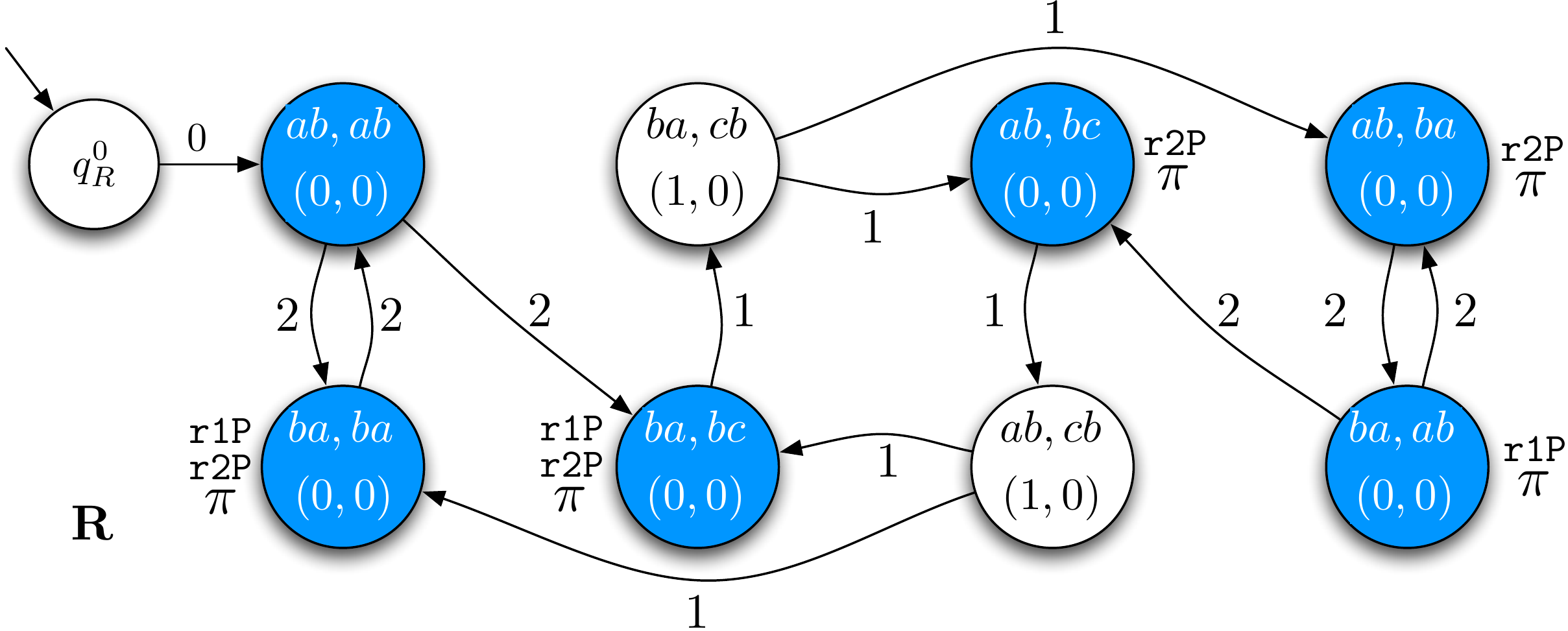}
	\caption{Region automaton obtained using \constR \cite{iros2011} that captures the joint behavior of the robotic team given in Fig.~\ref{fig:simple_ts}. $\mathtt{Sync}$ states where all robots occupy vertices, \ie states with all zero clock values, are highlighted in blue.}
	\label{fig:ra}
\end{figure}

If $\phi_{sync}$ is a trace-closed formula, we obtain the region automaton that captures the joint behavior of the robotic team using \constR \cite{iros2011}. 
Next, using Alg.~\ref{alg:serializeR}, we first introduce synchronization states by adding the special $\mathtt{Sync}$ proposition to the states where all robots occupy some vertex in their TS's simultaneously, \ie states with $r=(0,\ldots,0)$.
Note that, these are the states that will be used to calculate a bound on optimality when the robots are deployed in the field.
We then expand the states where multiple propositions are satisfied simultaneously to obtain $\BFR_{ser}$ where at most one proposition is satisfied at each state.
This ensures that languages of both the B\"uchi automaton that corresponds to $\phi_{sync}$ and $\BFR_{ser}$ are over $\Pi$.

\vskip 5pt
\exampler{\ref{exp:running_example}}{%
Fig.~\ref{fig:ra} illustrates the region automaton $\BFR$ that captures the joint behavior of the team given in Fig.~\ref{fig:simple_ts}.
The serialized region automaton with synchronization points $\BFR_{ser}$ that corresponds to $\BFR$ is given in Fig.~\ref{fig:ra_ser_sync}.%
}

\begin{figure}[h]
	\centering
	\includegraphics[width=1\linewidth]{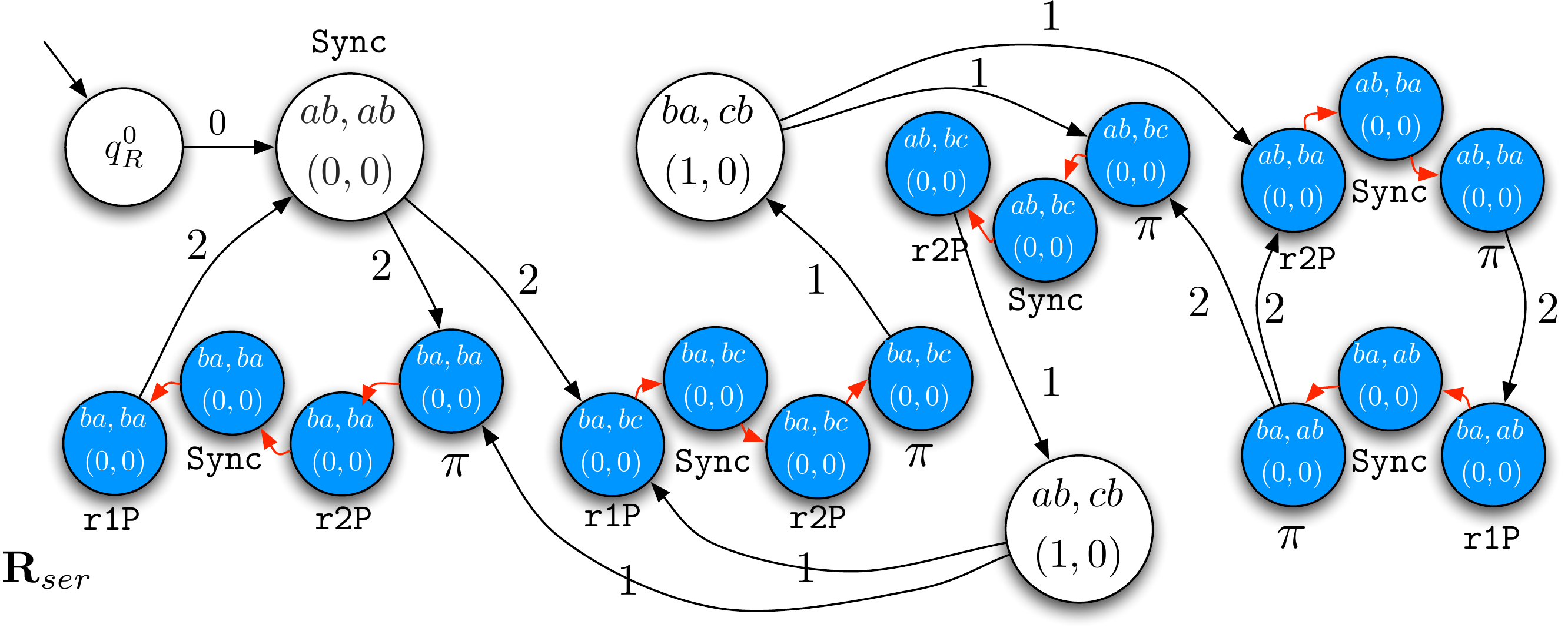}
	\caption{Serialized region automaton with synchronization states obtained after applying Alg.~\ref{alg:serializeR} to $\BFR$ in Fig.~\ref{fig:ra}. New states introduced after serialization are highlighted in blue. Red arrows stand for zero-weight transitions.}
	\label{fig:ra_ser_sync}
\end{figure}

\begin{algorithm}[h]
\DontPrintSemicolon
\SetInd{0.5em}{0.5em}
\KwIn{A region automaton $\BFR$ obtained using \constR.}
\KwOut{$\BFR_{ser}$, the serialized region automaton with synchronization states, such that at most one proposition is satisfied at each state.}
\BlankLine
\ForEach {State $\{q,r\}$ in $\BFR$} {
	\If {$r = (0,\ldots,0)$} {
		Add $\mathtt{Sync}$ to propositions satisfied in $\{q,r\}$.\;
	}
	$k \longleftarrow$ Number of propositions satisfied in $\{q,r\}$.\;
	\If {$k>1$} {
		$propsTuple \leftarrow$ The tuple $(\mathtt{p_1},\ldots,\mathtt{p_k})$ of propositions satisfied in $\{q,r\}$.\;
		Copy $\{q,r\}$ $k$ times to obtain $\{q,r\}'_1,\ldots,\{q,r\}'_k$.\;
		\ForEach {$i=1,\ldots,k$} {
			$\CL(\{q,r\}'_i) \leftarrow propsTuple[i]$.\;
			\If {$i<k$} {
				Add $\{q,r\}'_i \rightarrow \{q,r\}'_{i+1}$ to $\delta_R$ with zero weight.\;
			}
		}
		Re-direct all incoming transitions of $\{q,r\}$ to $\{q,r\}'_1$.\;
		Originate all outgoing transitions of $\{q,r\}$ from $\{q,r\}'_k$.\;
		Remove $\{q,r\}$ from $\CQ_R$.\;
	}
}
\caption{{\sc Serialize-Region-Automaton}\label{alg:serializeR}}
\end{algorithm}

\begin{remark}
Since $\phi_{sync}$ is trace-closed, the serialization can be done in any order.
Since all possible orderings belong to the same trace-equivalent class, they do not affect the satisfaction of the formula.
\end{remark}

\subsection{Finding the Robust Optimal Run and the Optimality Bound\label{sec:sub:soln}}
After obtaining the serialized region automaton $\BFR_{ser}$, we find an optimal run $r_\RMR^\star$ on $\BFR_{ser}$ that minimizes the cost function \eqref{eqn:cost_function} using our earlier \optrun algorithm \cite{SLS-JT-CB-DR:10b}.
The optimal run $r_\RMR^\star$ is always in a prefix-suffix form (Def.~\ref{def:prefix_suffix}).
Furthermore, as $r_\RMR^\star$ satisfies $\phi_{sync}$, it has at least one synchronization point in its suffix cycle, which we assume to start with a synchronization point.

\begin{definition}[\bf Projection of a run on $\BFR$ to $\BFT_{i}$s]
\label{def:projection_of_runs}
Given $\BBT$ and the corresponding run $r_\RMR$ on $\BFR_{ser}$ where 
\begin{multline*}
r_\RMR=((q^{0}_{1}q^{1}_{1},\ldots,q^{0}_{m}q^{1}_{m}),(x_1^0,\ldots,x_m^0))\\((q^{1}_{1}q^{2}_{1},\ldots,q^{1}_{m}q^{2}_{m}),(x_1^1,\ldots,x_m^1))\ldots,
\end{multline*}
we define its projection on $\BFT_{i}$ as run $r_{i}=q^{0}_{i}q^{1}_{i}\ldots$ for all $i=1,\ldots,m$, where $q^{k}_{i}$ only appears in $r_{i}$ if $x_i^k=0$ and $\BBT(k) \not = \BBT(k+1)$.
\end{definition}

In \cite{iros2011} we show that the individual runs $r_i$ obtained by the projection in Def.~\ref{def:projection_of_runs} are equivalent to the region automaton run $r_\RMR$ in the sense that they produce the same word $\omega_{team}$.
Using Def.~\ref{def:projection_of_runs}, we project the optimal run $r_\RMR^\star$ to individual $\BFT_i$s to obtain the set of optimal individual runs $\{r_1^\star,\ldots,r_m^\star\}$.
As the robots execute their infinite runs in the field, they synchronize with each other at the synchronization point following the protocol given in Alg.~\ref{alg:syncrun} ensuring that they start each new suffix cycle in a synchronized way.
Using this protocol, we can define a bound on optimality, \ie the value of the cost function \eqref{eqn:cost_function} observed in the field, as given in the following proposition.

\begin{algorithm} 
\DontPrintSemicolon 
\SetInd{0.5em}{0.5em}
\KwIn{A run $r_k$ of robot $k$ in the prefix-suffix form with at least one synchronization point in its suffix cycle.}
\Begin{
	$syncPoint \leftarrow$ First synchronization point in the suffix.\;
	$teamFlags \leftarrow (0,\ldots,0)$.\;
	\While {True} {
		\If {$syncMessage$ received from robot $i$} {
			$teamFlags[i] \leftarrow 1$.\;	
		}
		\If {$currentState = syncPoint$} {
			Stop\;
			Broadcast $syncMessage$.\;
			$teamFlags[k] \leftarrow 1$.\;	
		}
		\If {$teamFlags = (1,\ldots,1)$} {
			$teamFlags \leftarrow (0,\ldots,0)$.\;
			Continue executing $r_k$.\;
		}
	}
}
\caption{{\sc Sync-Run}\label{alg:syncrun}}
\end{algorithm}

\begin{proposition}
\label{prp:bound_on_opt}
Suppose that each robot's deviation value is bounded by $\rho >0$ (i.e., $\rho_i \leq \rho$ for all robots $i$), and let $J(\BBT^\pi)$ be the cost of the planned robot paths.
Then, if the robots follow the protocol given in Alg.~\ref{alg:syncrun} the field value of the cost satisfies
\[
\overline{J(\BBT^\pi)} \leq J(\BBT^\pi) + \rho(J(\BBT^\pi)+ 2 d_s),
\]
where $d_s$ is the planned duration of the suffix cycle.
\end{proposition}
\begin{proof}
In the following, we take the suffix to begin at a synchronization point.
The suffix consists of an infinite number of repetitions of the suffix cycle, which we denote $S_{\mathrm{c}}$.
Let $d_s$ be the planned duration of $S_{\mathrm{c}}$, let $n_s$ be the number of optimizing propositions satisfied in $S_{\mathrm{c}}$.
Let us redefine $t=0$ to be the time when the suffix starts, and let $\bar\BBT^\pi$ be a sequence of length $n_s$ recording the $n_s$ times that the optimizing proposition is satisfied on the first repetition of $S_{\mathrm{c}}$.
Note that, as we consider infinite runs and as the process restarts itself at the beginning of each $S_{\mathrm{c}}$ by means of the synchronization protocol given in Alg.~\ref{alg:syncrun}, we only need to consider the first repetition of $S_{\mathrm{c}}$.
We first define
\begin{equation*}
\begin{aligned}
\underline{T^i} &= \bar\BBT^\pi(i)(1-\rho)\\
\overline{T^i} &= \bar\BBT^\pi(i)(1+\rho)\\
t^{w} &= d_s(1+\rho)
\end{aligned}
\end{equation*}
where, $\underline{T^i}$ and $\overline{T^i}$ are the earliest and latest times that the $i$th optimizing proposition can be satisfied, respectively.
The value $t^w$ is the latest time that the second repetition of $S_{\mathrm{c}}$ can begin.
Then, for $0<i\leq n_s$, the worst-case time between satisfying the $i$th optimizing proposition and the $(i+1)$th optimizing proposition is
\begin{equation}
\label{eqn:tau_1}
\tau^{i,i+1} = 
\begin{cases} 
\overline{T^{i+1}} - \underline{T^i} & \text{if $0<i<n_s$,} \\
t^w + \overline{T^{1}} - \underline{T^{n_s}} & \text{if $i = n_s$}.\\
\end{cases}
\end{equation} 

Next, in the planned paths, multiple robots may simultaneously satisfy the $i$th optimizing proposition.
In the field, these satisfactions will not occur simultaneously.
The maximum amount of time between the first and last of these satisfying instances for the $i$th proposition, for $0<i\leq n_s$, is
\begin{equation}
\label{eqn:tau_2}
\tau^i = \overline{T^{i}} - \underline{T^i}.
\end{equation} 

Finally, using \eqref{eqn:tau_1} and \eqref{eqn:tau_2} we obtain the upper bound on the value of the cost function \eqref{eqn:cost_function} that will be observed during deployment as
\begin{equation}
\label{eqn:cost_max}
\overline{J(\BBT^\pi)} = \max\{\max_i\{\tau^{i,i+1}\},\max_i\{\tau^i\}\}.
\end{equation}

Substituting the definitions for $\overline{T^{i}}$, $\underline{T^{i}}$, and $t^w$ into~\eqref{eqn:tau_1} we obtain $\tau^{i,i+1} =$
\begin{multline*}
\begin{cases} 
\bar\BBT^\pi(i+1)(1+\rho)- \bar\BBT^\pi(i)(1-\rho) & \text{if $0<i<n_s$,} \\
(1+\rho)\big(d_s + \bar\BBT^\pi(1)\big) - \bar\BBT^\pi(n_s)(1-\rho) & \text{if $i = n_s$}\\
\end{cases}
\end{multline*}
But, we have that $ J(\BBT^\pi) \geq \bar\BBT^\pi(i+1) - \bar\BBT^\pi(i)$, and $ J(\BBT^\pi) \geq d_s + \bar\BBT^\pi(1) - \bar\BBT^\pi(n_s)$.
In addition, $\bar\BBT^\pi(1) \leq J(\BBT^\pi)$ and $\bar\BBT^\pi(i) \leq d_s$ for all $i\in\{2,\ldots,n_s\}$.
Using these expressions we obtain
\[
\tau^{i,i+1} \leq  J(\BBT^\pi) + \rho(J(\BBT^\pi)+ 2 d_s).
\]
Similarly, we get
\[
\tau^{i} \leq 2\rho d_s,
\]
and thus $\overline{J(\BBT^\pi)} \leq J(\BBT^\pi) + \rho(J(\BBT^\pi)+ 2 d_s)$.  
\end{proof}

\begin{remark}
In Prop.~\ref{prp:bound_on_opt}, we have provided a conservative bound for ease of presentation.
However, we can calculate an exact bound on the field value of the cost $J(\BBT^\pi)$ using a treatment similar to the proof of Prop~\ref{prp:bound_on_opt}.
\end{remark}

\exampler{\ref{exp:running_example}}{%
For the example we have shown throughout this section, applying Alg. \optrun \cite{SLS-JT-CB-DR:10b} to $\BFR_{ser}$ given in Fig.~\ref{fig:ra} and the formula $\phi_{sync}:= \Always \Event \mathtt{r1P} \andltl \Always \Event \mathtt{r2P} \andltl \Always \Event \pi \andltl \Always \Event \mathtt{Sync}$ results in the optimal run with the prefix
\begin{center}
\scalebox{0.83}{%
\begin{tabular}{c |c c c c c c}
$\BBT$ & 0 & 2 & 2 & 2 & 2 & 3 \\
\hline
\multirow{2}{*}{$r_\RMR^\star$} & ab,ab & ba,bc & ba,bc & ba,bc & ba,bc & ba,cb \\
& (0,0) & (0,0) & (0,0) & (0,0) & (0,0) & (1,0) \\
\hline
$\CL_\BFR(\cdot)$ & $\mathtt{Sync}$ & $\mathtt{r1P}$ & $\mathtt{Sync}$ & $\mathtt{r2P}$ & $\opt$ & \\
\end{tabular}}
\end{center}
and the suffix cycle
\begin{center}
\scalebox{0.83}{%
\begin{tabular}{c |c c c c c c }
$\BBT$& 4 & 4 & 4 & 6 & 6 & 6\\
\hline
\multirow{2}{*}{$r_\RMR^\star$} & ab,ba & ab,ba & ab,ba & ba,ab & ba,ab & ba,ab \\
& (0,0) & (0,0) & (0,0) & (0,0) & (0,0) & (0,0) \\
\hline
$\CL_\BFR(\cdot)$ & $\mathtt{r2P}$ & $\mathtt{Sync}$ & $\opt$ & $\mathtt{r1P}$ & $\mathtt{Sync}$ & $\opt$ \\
\end{tabular}}
\end{center}
which will be repeated an infinite number of times.
In the table above, the rows correspond to the times when transitions occur, the run $r_\RMR^\star$, and the satisfying atomic propositions, respectively.
For this example, $\BBT^\pi = 2,4,6,8,10,\ldots $ and the cost as defined in \eqref{eqn:cost_function} is $J(\BBT^\opt) = 2$.
Furthermore, when the robotic team is deployed in the field, this cost is bounded from above by 2.5 for $\rho_1=\rho_2=0.05$ as given by Prop.~\ref{prp:bound_on_opt}.

Applying Def.~\ref{def:projection_of_runs} to $r_\RMR^\star$ we have the following individual runs:
\begin{center}
\scalebox{0.83}{%
\begin{tabular}{ c |c c c c c c c c}
$\BBT$& 0 & 2 & 3 & 4 & 6& 8 & 10 & \ldots \\
\hline
$r_1^\star$ & a & b &   & a & b & a & b & \ldots \\
\hline
$r_2^\star$ & a & b & c & b & a & b & a & \ldots
\end{tabular}}
\end{center}
Note that, at time $t=3$, the second robot has arrived at $c$ while the first robot is still traveling from $b$ to $a$, therefore the clock of the first robot is not zero at this time, \ie $x_{1}\neq 0$, and $b$ does not appear in $r_{1}^\star$ at time $t=3$.%
}

We finally summarize our approach in Alg.~\ref{alg:robustmultioptrun}, show that this algorithm indeed gives a solution to Prob. \ref{prb:problem} and analyze the overall complexity of our approach.

\begin{algorithm} 
\DontPrintSemicolon 
\SetInd{0.5em}{0.5em}
\KwIn{$m$ $\BFT_i$'s and a global LTL specification $\phi_{sync}$ of form \eqref{eqn:general_formula}.}
\KwOut{A set of robust optimal runs $\{r^{\star}_1,\ldots, r^{\star}_{m}\}$ that satisfies $\phi_{sync}$, minimizes \eqref{eqn:cost_function}, and the bound on the performance of the team in the field.}
\Begin{
	$\phi_{sync} := \varphi \andltl \Always \Event \opt \andltl \Always \Event \mathtt{Sync}$.\;
	\If {$\phi_{sync}$ $\mathrm{is\;trace}\text{-}\mathrm{closed}$} {
		Obtain the region automaton $\BFR$ using \constR \cite{iros2011}.\;
		Obtain $\BFR_{ser}$ using \serializeR.\;
		Find the optimal run $r_\RMR^\star$ applying \optrun \cite{SLS-JT-CB-DR:10b} to $\BFR_{ser}$ and $\phi_{sync}$.\;
		Obtain individual runs from $r_\RMR^\star$ using Def.~\ref{def:projection_of_runs}.\;
		Find the bound on optimality as given in Prop.~\ref{prp:bound_on_opt}.\;
	}
	\Else {
		Abort.\;
	}
}
\caption{{\sc Robust-Multi-Robot-Optimal-Run}\label{alg:robustmultioptrun}}
\end{algorithm}

\begin{proposition}
\label{prp:final}
Alg.~\ref{alg:robustmultioptrun} solves Prob.~\ref{prb:problem}.
\end{proposition}
\begin{proof}
Note that Alg.~\ref{alg:robustmultioptrun} combines all steps outlined in this section.
The planned word $\omega_{team}$ generated by the entire team satisfies $\phi_{sync}$ and minimizes \eqref{eqn:cost_function}, as shown in \cite{iros2011}.
Furthermore, since $\phi_{sync}$ is trace-closed, the optimal satisfying run is guaranteed not to violate $\phi_{sync}$ in the field due to timing errors as given in Prop.~\ref{prp:trace_closed_formula}.
Therefore, $\{r^{\star}_1,\ldots, r^{\star}_{m}\}$ as obtained from Alg.~\ref{alg:robustmultioptrun} is the solution to Prob.~\ref{prb:problem}.
\end{proof}

\begin{proposition}
\label{prp:complexity}
For the case where a group of $m$ identical robots are expected to satisfy an LTL specification $\phi$ in a common environment with $\Delta$ edges and a largest edge weight of $W$, the worst-case complexity of Alg.~\ref{alg:robustmultioptrun} is $O((\Delta \cdot W)^{3m}\cdot2^{O(|\phi|)})$.
\end{proposition}
\begin{proof}
From \cite{iros2011}, the number of states of the region automaton $\BFR$ is bounded by
\[
\left(\prod_{i=1}^m |\delta_i|\right)\left(\prod_{i=1}^m W_i - \prod_{i=1}^m (W_i-1)\right)+1
\]
where $m$ is number of robots and $W_i$ is largest edge weight in TS $\BFT_i$ of robot $i$.
Then, for the above mentioned case, the worst-case size of the region automaton is $O((\Delta \cdot W)^m)$ .
In \cite{SLS-JT-CB-DR:10b-IJRR}, the authors give the worst-case complexity of the \optrun algorithm as $O(|T|^3\cdot2^{O(|\phi|)})$ where $|T|$ is the number of states of the input transition system and $|\phi|$ is the length of the LTL specification.
Therefore, the worst-case complexity of Alg.~\ref{alg:robustmultioptrun} becomes $O((\Delta \cdot W)^{3m}\cdot2^{O(|\phi|)})$.
\end{proof}

\section{Implementation and Case Studies \label{sec:experiments}}

We implemented Alg.~\ref{alg:robustmultioptrun} in objective-C as the software package {\sc LTL Robust Optimal Multi-robot Planner (LROMP)} and used it in conjunction with our earlier \optrun \cite{SLS-JT-CB-DR:10b} algorithm to obtain robust and optimal trajectories for robots performing persistent data gathering missions in a road network environment.
The software package, available at \url{http://hyness.bu.edu/Software.html}, utilizes the dot tool \cite{dotTool} to visualize transition systems and the \optrun algorithm uses the LTL2BA software \cite{LTL2BA} to convert LTL specifications to B\"uchi automata.
Following the steps detailed in Sec.\ref{sec:solution}, the software first creates the serialized region automaton with synchronization states $\BFR_{ser}$ using $\BFT_i$s defined by the user and exports an M-file which defines $\BFR_{ser}$ in Matlab.
Next, $\phi_{sync}$ is checked for trace-closedness, after which \optrun algorithm is executed in Matlab to find the optimal run $r_\RMR^\star$ on $\BFR_{ser}$.
Finally, an upper bound on the field value of the cost function \eqref{eqn:cost_function} is computed and $r_\RMR^\star$ is projected to individual $\BFT_i$, $i\!=\!1,\ldots,m$, to obtain the solution to Prob.~\ref{prb:problem}.

\begin{figure}
\subfloat[]{\includegraphics[width=0.7\linewidth,angle=270]{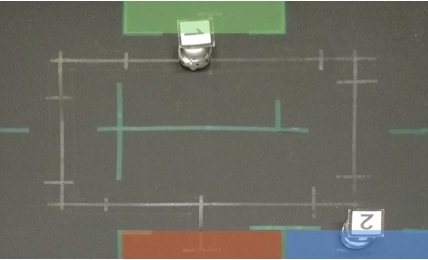}\label{}}\hspace{0.05in}
\subfloat[]{\raisebox{-0.15in}{\includegraphics[width=0.59\linewidth,angle=270]{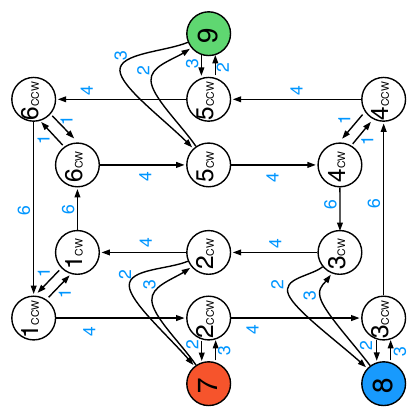}}\label{fig:new_rule_ts}}
\caption{ (a) Road network used in the experiments (b) The model of the road network with weights shown in blue. 1 time unit in this model corresponds to 3 seconds. The red and blue regions are data gathering locations of robots 1 and 2, respectively and the green region is the common upload location. CW and CCW stand for clockwise and counter-clockwise, respectively.}
\label{fig:road_network}
\end{figure}
 
Fig.~\ref{fig:road_network} illustrates our experimental platform, which is a road network consisting of roads, intersections, and task locations. The figure also shows the transition system that models the motion of the robots on this road network where 1 time unit corresponds to 3 seconds.
In the following, the transition systems $\BFT_{i}$ are identical except for their initial states and the sets of propositions that can be satisfied at states.

In our experiments, we consider a persistent monitoring task where two robots with deviation values of $\rho_1=0.09$, $\rho_2=0.04$ repeatedly gather and upload data and the maximum time in between any two data uploads must be minimized.
We require robots 1 and 2 to gather data at 7 and 8 in Fig.~\ref{fig:road_network}, respectively and upload the data at 9. We define
$\Pi~\!=\!~\{\mathtt{R1Gather},\;\mathtt{R1Upload},\;\mathtt{R2Gather},\;\mathtt{R2Upload},\;\mathtt{Upload},$ $\mathtt{Sync}\}$
and assign the atomic propositions as
\begin{align*}
\mathcal L_{1}(7)&=\{\mathtt{R1Gather}\}, \CL_1(9)=\{\mathtt{R1Upload},\mathtt{Upload}\}\\
\mathcal L_{2}(8)&=\{\mathtt{R2Gather}\}, \CL_2(9)=\{\mathtt{R2Upload},\mathtt{Upload}\}.
\end{align*}
where $\mathtt{Upload}$ is set as the optimizing proposition ($\pi$ as in formula \eqref{eqn:general_formula}) due to the task specification.
Next, we forbid data uploads unless robots have something to upload using the LTL formula
\begin{multline*}
\varphi = \Always(\mathtt{R1Upload} \Implies \Next(\Not \mathtt{R1Upload}\ \Until\ \mathtt{R1Gather})) \\
\andltl \Always(\mathtt{R2Upload} \Implies \Next(\Not \mathtt{R2Upload}\ \Until\ \mathtt{R2Gather})).
\end{multline*}
Our overall LTL formula in the form of \eqref{eqn:general_formula} is 
\begin{equation}
\label{eqn:experiments}
\phi_{sync}=\varphi \andltl \Always\,\Event\, \mathtt{Upload} \andltl \Always\,\Event\,\mathtt{Sync}.
\end{equation}
 
\begin{figure}
\centering
\includegraphics[width=0.7\linewidth]{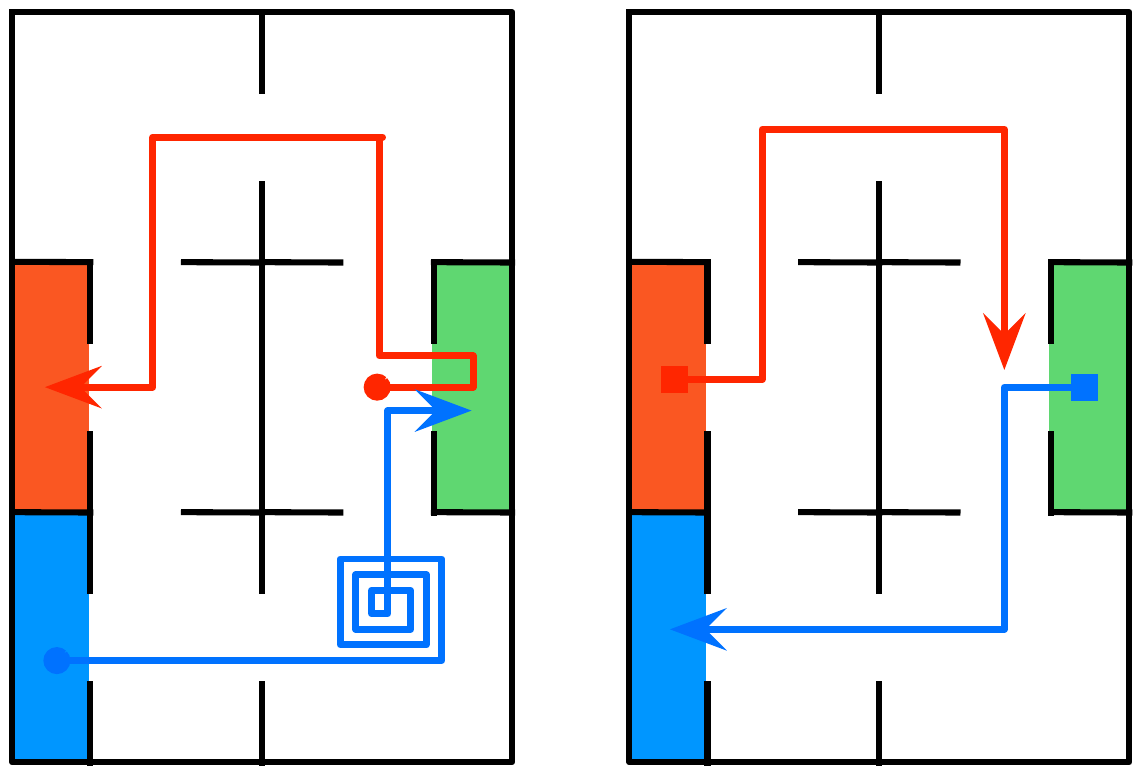}
\caption{Team trajectories used in the experiments. The red and blue regions are data gathering locations of robots 1 and 2, respectively and the green region is the common upload location. The circles on the left show the sync point, \ie the beginning of the suffix cycle, on the trajectories of the robots.}
\label{fig:run}
\end{figure}
Running our algorithms on an iMac i5 quad-core computer, we obtain the robust optimal trajectory as illustrated in Fig.~\ref{fig:run}.
The algorithm ran for 35 minutes, and the region automaton $\BFR_{ser}$ had 5224 states.
The value of the cost function was 19 time units (57 seconds) with an upper-bound of 27.55 time units (82.65 seconds), meaning that the maximum time in between data uploads would be less than 82.65 seconds in the field.
This result was experimentally verified in our robotic test-bed and the maximum time in between data uploads was measured to be 64 seconds (21.3 time units) during a run of 13 minutes.
In order to confirm and demonstrate the effectiveness of our approach, we executed the same trajectory without any synchronization.
After approximately 6.5 minutes, the maximum time in between data uploads was measured to be 92 seconds (30.7 time units), much worse than what is provided by our approach.
Our video submission accompanying the paper displays the robot trajectories for both cases.

It is interesting to note that, in the optimal solution the second robot spends extra time spinning between states $4_{\RMC\RMW}$ and $4_{\RMC\RMC\RMW}$ (Figs.~\ref{fig:new_rule_ts}, \ref{fig:run}).
This behavior is actually time-wise optimal as it decreases the maximum time between successive satisfying instances of the optimizing proposition.
 
\section{Conclusions \label{sec:conclusions}}
In this paper we presented and experimentally evaluated a method for planning robust optimal trajectories for a team of robots that satisfy a common temporal logic mission specification.
Our method is robust to uncertainties in the traveling times of each robot, and thus has practical value in applications where multiple robots must perform a series of tasks collectively in a common environment.
We considered trace-closed temporal logic formulas with optimizing and synchronizing propositions that must be repeatedly satisfied.
In the absence of timing errors, the motion plan delivered by our method is optimal in the sense that it minimizes the maximum time between satisfying instances of the optimizing proposition.
If the traveling times observed in the field deviate from those given by the transition systems of the robots, our method guarantees that the mission specification is never violated and provides an upper bound on the ratio between the performance in the field and the optimal performance. 

\section*{Acknowledgments}
We thank Jennifer Marx at Boston University for her work on the experimental platform.

\bibliographystyle{IEEEtran} 
\bibliography{references}

\end{document}